\documentclass[twoside,11pt]{article}
\usepackage{pgm} 

\usepackage{pifont}

\usepackage{graphicx}
\usepackage{subfigure} 
\usepackage{natbib}
\usepackage{algorithm}
\usepackage{algorithmic}

\usepackage{helvet}
\usepackage{courier}
\usepackage{float}
\usepackage{comment}

\usepackage{url}
\usepackage{amsmath}
\usepackage{tabularx}
\usepackage{multirow}
\usepackage{multicol}
\usepackage{graphicx}
\usepackage{caption}
\usepackage{subfigure}
\usepackage{enumitem}
\graphicspath{ {figures/} }

\newcommand{\tuple}[1]{\ensuremath{\left \langle #1 \right \rangle }}

\ShortHeadings{Dynamic Sum Product Networks for Tractable Inference on Sequence Data}{Melibari, Poupart, Doshi, and Trimponias}

\begin{document}
\title{Dynamic Sum Product Networks for \\Tractable Inference on Sequence Data \\ (Extended Version)}

\author{\name Mazen Melibari \email mmelibar@uwaterloo.ca \\
		\name Pascal Poupart \email ppoupart@uwaterloo.ca \\
       \addr David R. Cheriton School of Computer Science \\
       University of Waterloo \\ 
       Waterloo, Ontario, Canada 
       \AND
       \name Prashant Doshi \email pdoshi@cs.uga.edu \\
       \addr Department of Computer Science \\
       University of Georgia \\ 
       Athens, Georgia, USA
		\AND
       \name George Trimponias \email g.trimponias@huawei.com \\
       \addr Huawei Noah's Ark Lab \\
       Hong Kong}
\editor{Alessandro Antonucci, Giorgio Corani, and Cassio Polpo de Campos}
\maketitle

\begin{abstract}%
Sum-Product Networks (SPN) have recently emerged as a new class of tractable probabilistic models. Unlike Bayesian networks and Markov networks where inference may be exponential in the size of the network, inference in SPNs is in time linear in the size of the network. Since SPNs represent distributions over a fixed set of variables only, we propose dynamic sum product networks (DSPNs) as a generalization of SPNs for sequence data of varying length. A DSPN consists of a template network that is repeated as many times as needed to model data sequences of any length. We present a local search technique to learn the structure of the template network.  In contrast to dynamic Bayesian networks for which inference is generally exponential in the number of variables per time slice, DSPNs inherit the linear inference complexity of SPNs.  We demonstrate the advantages of DSPNs over DBNs and other models on several datasets of sequence data. 

\end{abstract}

\section{Introduction}

Probabilistic graphical models~\citep{koller2009probabilistic} such as Bayesian networks (BNs) and Markov netwoks (MNs) provide a general framework to represent multivariate distributions while exploiting conditional independence.  Over the years, many approaches have been proposed to learn the structure of those networks~\citep{neapolitan2004learning}.  However, even if the resulting network is small, inference may be intractable (e.g., exponential in the size of the network) and practitioners must often resort to approximate inference techniques.  Recent work has focused on the development of alternative probabilistic models such as arithmetic circuits (ACs)~\citep{darwiche2003differential} and sum-product networks (SPNs)~\citep{poon2011sum} for which inference is guaranteed to be tractable (e.g., linear in the size of the network for SPNs and ACs).  This means that the networks learned from data can be directly used for inference without any further approximation.  So far, this work has focused on learning models for a fixed number of variables based on fixed-length data~\citep{lowd2012learning,dennis2012learning,gens2013learning,peharz2013greedy,rooshenas2014learning}.  

We present Dynamic Sum-Product Networks (DSPNs) as an extension to SPNs that model sequence data of varying length.  Similar to Dynamic Bayesian networks (DBNs)~\citep{dean1989model}, DSPNs consist of a {\em template network} that repeats as many times as the length of a data sequence.  We describe an invariance property for the template network that is sufficient to ensure that the resulting DSPN is valid (i.e., encodes a joint distribution) by being complete and decomposable.  Since existing structure learning algoritms for SPNs assume a fixed set of variables and fixed-length data, they cannot be used to learn the structure of a DSPN.  We propose a general anytime search-and-score framework with a specific local search technique to learn the structure of the template network that defines a DSPN based on data sequences of varying length.  We demonstrate the advantages of DSPNs over static SPNs, DBNs, hidden Markov models (HMMs) and recurrent neural networks (RNNs) with synthetic and real sequence data.
\section{Background}
\label{sec:background}




\begin{definition}[Sum-Product Network~\citep{poon2011sum}]
A sum-product network (SPN) over $n$ binary variables $X_1,...,X_n$ is a rooted directed acyclic graph whose leaves are the indicators $I_{x_1},...,I_{x_n}$ and $I_{\bar{x}_1},...,I_{\bar{x}_n}$, and whose internal nodes are sums and products. Each edge $(i, j)$ emanating from a sum node $i$ has a non-negative weight, $w_{ij}$. The value of a product node is the product of the values of its children. The value of a sum node is $\sum_{j\in Ch(i)} w_{ij} v_j$, where $Ch(i)$ is the set of children of $i$ and $v_j$ is the value of node $j$. The value of an SPN is the value of its root.
\end{definition}

The value of an SPN can be seen as the output of a network polynomial whose variables are the indicator variables and the coefficients are the weights~\citep{darwiche2003differential}. This polynomial represents a joint probability distribution over the variables if the SPN is {\em valid}.  {\em Completeness} and {\em decomposability} (see below) are sufficient conditions for validity~\citep{darwiche2003differential, poon2011sum} that impose some conditions on the {\em scope} of each node, which is the set of variables that appear in the sub-SPN rooted at that node.  

\begin{definition}[Completeness]
An SPN is complete \textit{iff} all children of the same sum node have the same scope.
\end{definition}
\begin{definition}[Decomposability]
An SPN is decomposable iff all children of the same product node have disjoint scopes.
\end{definition}



Several basic distributions can be encoded by simple SPNs. For instance, a univariate distribution can be encoded using an SPN whose root node is a sum that is linked to each indicator of a single variable $X$ (Fig.~\ref{fig:univirate-spn}).  A factored distribution over a set of variables $X_1,...,X_n$ is encoded by a root product node linked to univariate distributions for each variable $X_i$ (Fig.~\ref{fig:factored-spn}).  A naive Bayes model is encoded by a root sum node linked to a set of factored distributions (Fig.~\ref{fig:naivebayes-spn}) and a product of naive Bayes models is encoded by a root product node linked to a set of naive Bayes models (Fig.~\ref{fig:to-random-part-b}).

Inference queries $\Pr(X=x|Y=y)$ can be answered by taking the ratio of the values obtained by two bottom up passes of an SPN. In the first pass, we initialize $I_{x}=1$, $I_{\bar{x}}=0$, $I_{y}= 1$, $I_{\bar{y}}=0$ and set all remaining indicators to 1 in order to compute a value proportional to the desired query.  In the second pass, we initialize $I_{y}= 1$, $I_{\bar{y}}=0$ and set all remaining indicators to 1 in order to compute the normalization constant.  The linear complexity of inference in SPNs is an appealing property given that inference for other models such as BNs is exponential in the size of the network in the worst case.


While SPNs are computational graphs based on which it is difficult to infer the relationships between the variables (e.g., conditional independence), \cite{zhao2015spnbn} showed how to convert SPNs into equivalent bipartite Bayesian networks without any exponential blow up.  In contrast, the compilation of a Bayesian network into an equivalent SPN may yield an exponential blow up.  SPNs are syntactically equivalent to arithmetic circuits (ACs)~\citep{park2004differential} in the sense that they can be reduced to each other in linear time and space.     

An extension to network polynomials for dynamic Bayesian networks (DBNs) was given in \citep{brandherm2004extension}. A procedure based on variable elimination is proposed to compile a DBN into a recursive network polynomial that can be represented by a special AC that we call dynamic AC. 
 Since there is a risk that the compiled dynamic AC will be intractable,
the authors use the Boyen-Koller~\citep{boyen1998tractable} method to approximate the output with a factored representation.
Thus, compiling a DBN to a dynamic AC does not reduce the complexity of inference, but only makes it linear in the size of the compiled dynamic AC, which could be intractable. In contrast, we propose an approach to learn tractable models directly from sequence data. 
\begin{figure}[!tb]
\centering
	\subfigure[]{
		\begin{minipage}[b]{0.10\columnwidth}
				\centering
				\includegraphics[width=\textwidth]{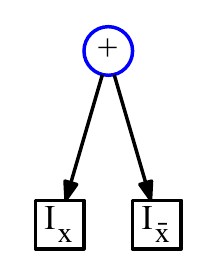}		
		\end{minipage}
		\label{fig:univirate-spn}
	}
	~
	\subfigure[]{
		\begin{minipage}[b]{0.24\columnwidth}	
				\centering					
				\includegraphics[width=\textwidth]{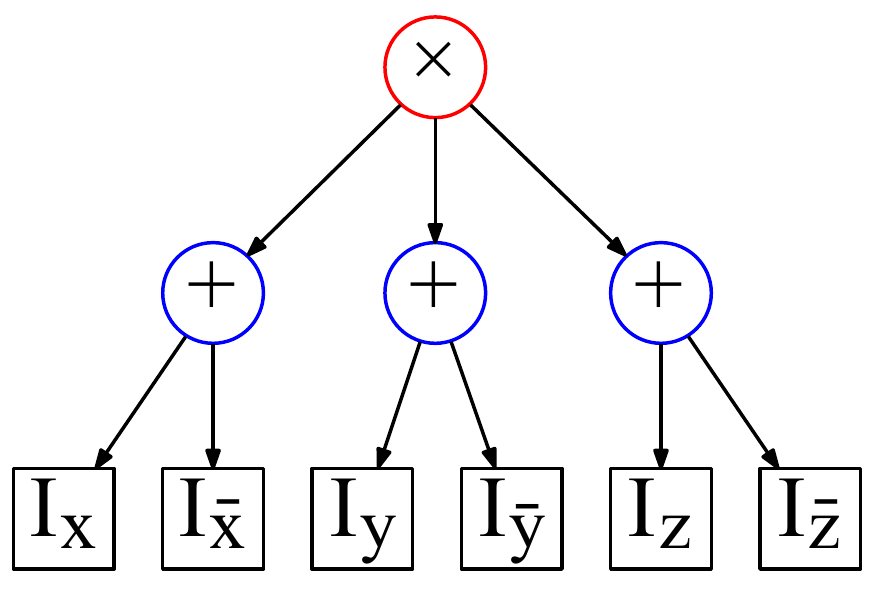}
		\end{minipage}
		\label{fig:factored-spn}
	}
	~
	\subfigure[]{
		\begin{minipage}[b]{0.24\columnwidth}
				\centering					
				\includegraphics[width=\textwidth]{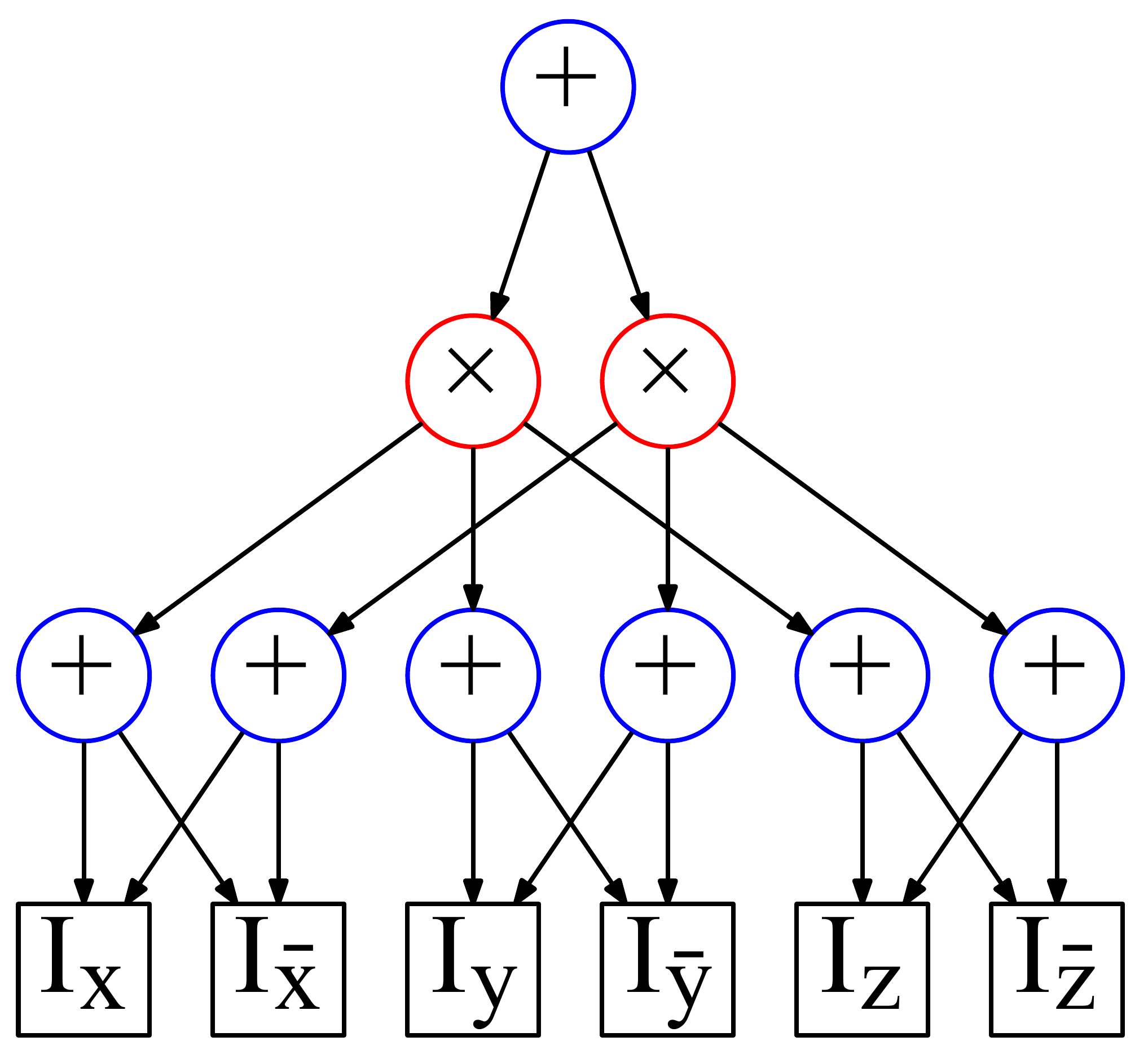}
		\end{minipage}
		\label{fig:naivebayes-spn}
	}	
    \subfigure[]{
		\begin{minipage}[t]{0.24\columnwidth}
			\centering		
				\includegraphics[width=\textwidth]{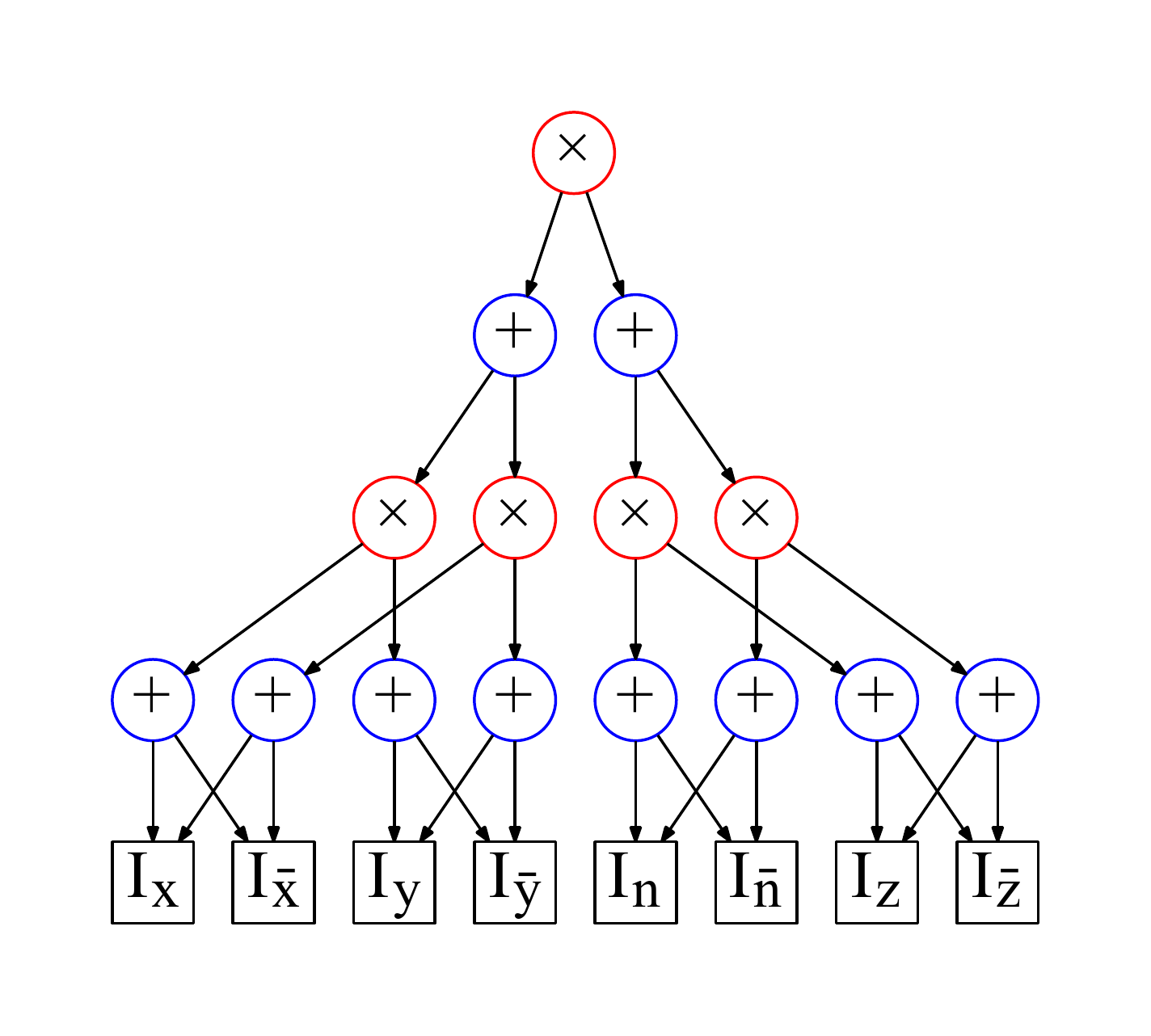}
		\end{minipage}
		\label{fig:to-random-part-b}
    }
\caption{(a) Univariate distribution over a binary variable $x$. (b) Factored distribution over three binary variables $x$, $y$, and $z$. (c) naive Bayes model over three binary variables $x$, $y$, and $z$. (d) Product of naive Bayes models.}
\label{fig:add-example}
\end{figure}
\section{Dynamic Sum-Product Networks}
\label{sec:DSPN}

We propose dynamic SPNs (DSPNs) as a generalization of SPNs for modeling sequence data of varying length. While DSPNs are equivalent to dynamic ACs (i.e., reducible to each other without any blow up), we develop a structure learning algorithm that learns a tractable DSPN directly from sequence data (instead of learning a DBN from data and then compiling it into a potentially exponentially larger DSPN or dynamic AC).  We also show sufficient conditions to ensure that estimated DSPNs are valid and therefore permit exact sequential inference in linear time.

Consider temporal sequence data that is generated by $n$ variables (or features) over $T$ time steps:  
\begingroup\makeatletter\def\f@size{8}\check@mathfonts
$\tuple{\tuple{X_1,X_2,\ldots,X_n}^1,\tuple{X_1,X_2,\ldots,X_n}^2,\ldots,\tuple{X_1,X_2,\ldots,X_n}^T}$
\endgroup
 where $X_i$, $i = 1 \ldots n$ is a random variable in one time slice and $T$ may vary with each sequence. Note that non-temporal sequence data such as sentences (sequence of words) can also be represented by sequences of repeated features. We will label the set of repeating variables as a {\em slice} and we will index slices by $t$ even if the sequence is not temporal, for uniformity.

A DSPN models sequences of varying length with a fixed number of parameters by using a template that is repeated at each slice. This is analogous to DBNs where the template corresponds to the network that connects two consecutive slices.  
\begin{definition}[Template network]
A template network for a slice of $n$ binary variables at time $t$, $\tuple{X_1,X_2,\ldots,X_n}^t$,  is a directed acyclic graph with $k$ roots and $k+2n$ leaf nodes. The $2n$ leaf nodes are the indicator variables, $I_{x_1^t},I_{x_2^t},\ldots,I_{x_n^t},I_{\bar{x}_1^t},I_{\bar{x}_2^t},\ldots,I_{\bar{x}_n^t}$.  The remaining $k$ leaves and an equal number of roots are interface nodes to and from the template for the previous and next slices, respectively. The interface and interior nodes are either sum or product nodes.  Each edge $(i, j)$ emanating from a sum node $i$ has a non-negative weight $w_{ij}$ as in a SPN. Furthermore, we define a bijective mapping $f$ between the input and output interface nodes.
\label{def:template} 
\end{definition}
Fig.~\ref{fig:initial-structure} shows a generic template network. In addition, we define two special networks.

\begin{figure}[!tb]
\centering
	\subfigure[]{
      \begin{minipage}{.3\textwidth}
        \centering
        \includegraphics[width=\linewidth]{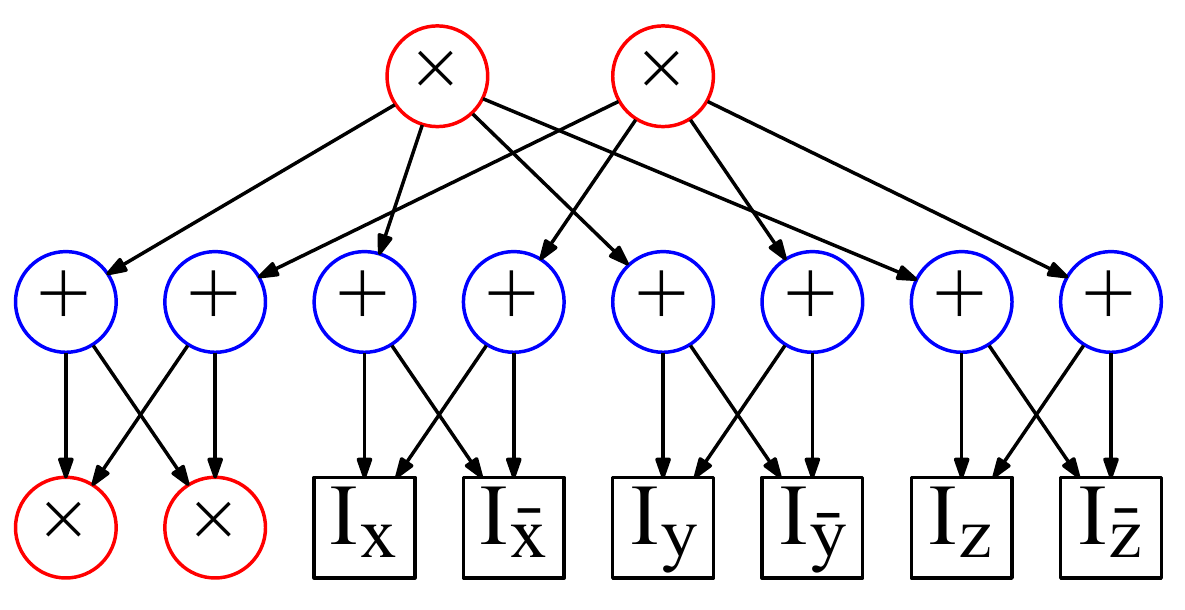}
      \end{minipage}%
		\label{fig:initial-structure}
	}
	~
	\subfigure[]{
      \begin{minipage}{.4\textwidth}
        \centering
        \includegraphics[width=\linewidth]{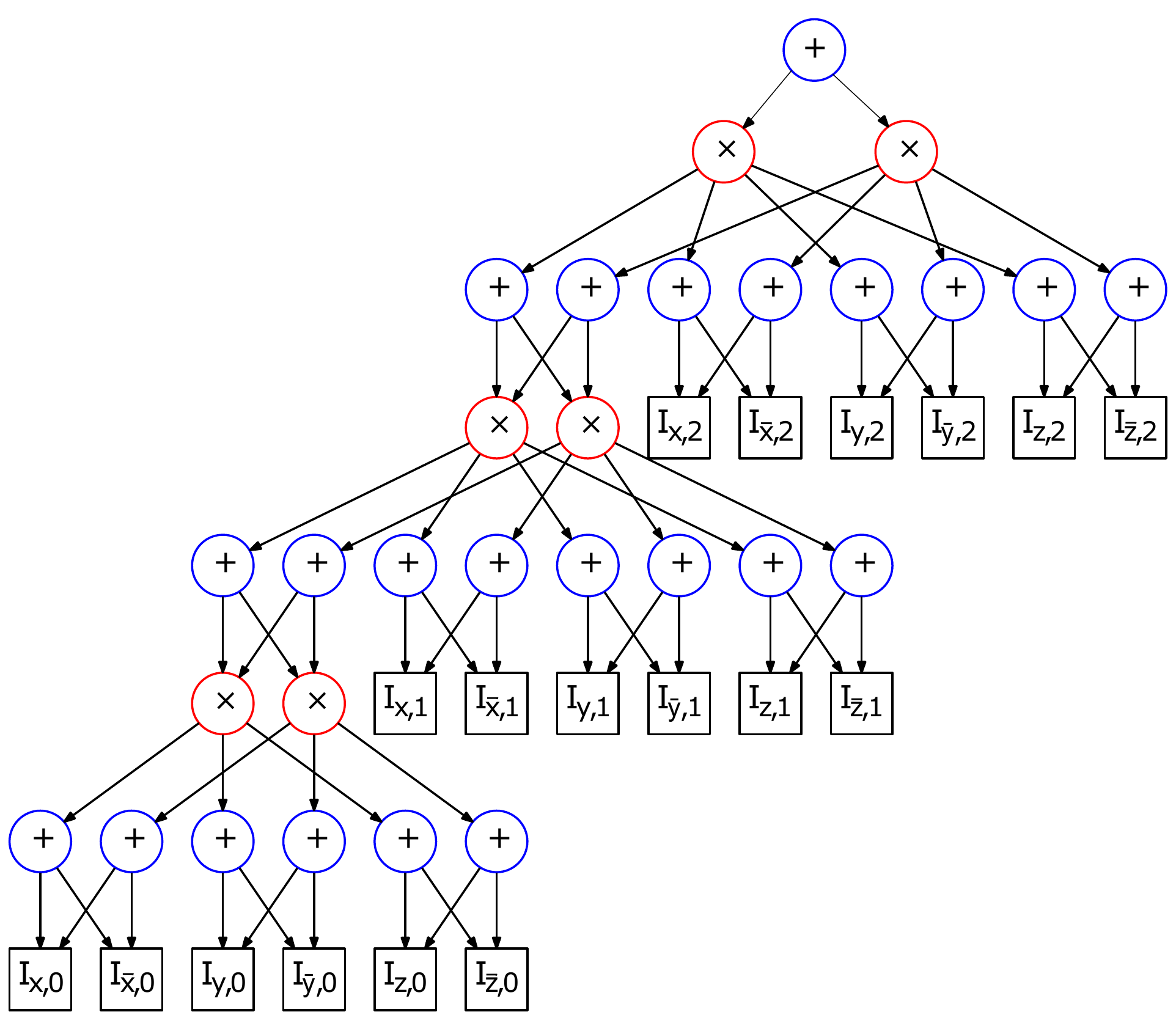}
      \end{minipage}
		\label{fig:unrolled-3slices}
	}
\caption{(a) An example of a generic template network. Notice the interface nodes in red. (b) A generic example of a complete DSPN unrolled over 3 time slices. Two template networks are stacked on the bottom network and capped by the top network.}
\label{}
\end{figure}



\begin{definition}[Bottom network]
A bottom network for the first slice of $n$ binary variables, \\ $\tuple{X_1, X_2, \ldots, X_n}^1$,  is a directed acyclic graph with $k$ roots and $2n$ leaf nodes. The $2n$ leaf nodes are the indicator variables, $I_{x_1^1},I_{x_2^1},\ldots,I_{x_n^1},I_{\bar{x}_1^1},I_{\bar{x}_2^1},...,I_{\bar{x}_n^1}$.  The $k$ roots are interface nodes to the template network for the next slice. The interface and interior nodes are either sum or product nodes.  Each edge $(i, j)$ emanating from a sum node $i$ has a non-negative weight $w_{ij}$ as in a SPN. 
\label{def:bottom}
\end{definition}

\begin{definition}[Top network]
Define a top network as a rooted directed acyclic graph composed of sum and product nodes with $k$ leaves. The leaves of this network are interface nodes, which were introduced previously. 
Each edge $(i, j)$ emanating from a sum node $i$ has a non-negative weight $w_{ij}$ as in a SPN. 
\label{def:top}
\end{definition}

Consider a data sequence of length $T$.  A DSPN of $T$ slices is obtained by stacking $T-1$ template networks of Def.~\ref{def:template} on top of a bottom network. This is capped by a top network. Two networks are stacked by merging the input interface nodes of the upper network with the output interface nodes of the lower network. 
Figure~\ref{fig:unrolled-3slices} shows an example with 3 slices of 2 variables each. 


As we mentioned previously, completeness and decomposability are sufficient to ensure the validity of an SPN. While one could check that each sum node in the DSPN is complete and each product node is decomposable, we provide a simpler way to ensure that any DSPN is complete and decomposable.
In particular,  we describe an invariance property for the template network that can be verified directly in the template without unrolling the DSPN. This invariance property is sufficient to ensure that completeness and decomposability are satisfied in the DSPN for any number of slices.

\begin{definition}[Invariance]
A template network over $\tuple{X_1,...,X_n}^t$ is invariant when the scope of each input interface node excludes variables $\{X_1^t,...,X_n^t\}$ and for all pairs of input interface nodes, $i$ and $j$, the following properties hold:
\begin{enumerate}[]
\item $scope(i) = scope(j) \vee scope(i) \cap scope(j) = \emptyset$
\item $scope(i) = scope(j) \iff scope(f(i)) = scope(f(j))$
\item 
$scope(i) \cap scope(j) =
\emptyset \iff scope(f(i)) \cap scope(f(j)) = \emptyset
$

\item all interior and output sum nodes are complete
\item all interior and output product nodes are decomposable
\end{enumerate}

Here $f$ is the bijective mapping that indicates which input nodes correspond to which output nodes in the interface. 
\label{def:invariance}
\end{definition}
Intuitively, a template network is invariant if we can assign a scope to each input interface node such that each pair of input interface nodes has the same scope or disjoint scopes, and the same relation holds between the scopes of the corresponding output nodes. Scopes of pairs of corresponding interface nodes must be the same or disjoint because a product node is decomposable when its children have disjoint scopes and a sum node is complete when its children have identical scope.  Hence, verifying the identity or disjoint relation of the scopes for every pair of input interface nodes helps us in verifying the completeness and decomposability of the remaining nodes in the template. Theorem~\ref{thm:DSPN_valid} below shows that the invariance property of Def.~\ref{def:invariance} can be used to ensure that the corresponding DSPN is complete and decomposable.

\begin{theorem}
If (a) the bottom network is complete and decomposable, (b) the scopes of all pairs of output interface nodes of the bottom network are either identical or disjoint, (c) the scopes of the output interface nodes of the bottom network can be used to assign scopes to the input interface nodes of the template and top networks in such a way that the template network is invariant and the top network is complete and decomposable, then the corresponding DSPN is complete and decomposable.
\label{thm:DSPN_valid}
\end{theorem}
\begin{proof}
We sketch a proof by induction (see Appendix~\ref{appendix} for more details).  For the base case, consider a single-slice DSPN (bottom network) capped with a top networks.  The bottom network is complete and decomposable by assumption.   Since the interface output nodes of the bottom network are merged with the input interface nodes of the top network, they are assigned the same scope, which ensures that the top network is also complete and decomposable.  
For the induction step, assume that a DSPN of $T$ slices is complete and decomposable.  Consider a DSPN of $T+1$ slices that shares the same bottom network and the same first $T-1$ copies of the template network as the DSPN of $T$ slices.  Hence the bottom network and the first $T-1$ copies of the template network in the DSPN of $T+1$ slices are complete and decomposable.  Since the next copy of the template network is invariant when its input interface nodes are assigned the scopes with the same identity and disjoint relations as the scopes of the output interface nodes of the bottom network, it is also complete and decomposable.  Similarly, the top network is complete and decomposable.
\end{proof}


\section{Structure Learning of DSPN}
\label{sec:structure}

As a DSPN is an SPN, we could ignore the repeated structure and learn an SPN for the number of variables corresponding to the longest sequence.  Shorter sequences could be treated as sequences with missing data for the unobserved slices. Unfortunately, this is intractable for very long sequences because the inability to model the repeated structure implies that the SPN will be very large and the learning computationally intensive. This approach may be feasible for datasets that contain only short sequences, nevertheless the amount of data needed may be prohibitively large because in the absence of a repeating structure the number of parameters is much higher. Furthermore, the SPN could be asked to perform inference on a sequence that is longer than any of the training sequences, and it is likely to perform poorly.  

Alternately, it is tempting to apply existing algorithms to learn the repeated structure of the DSPN.  Unfortunately, this is not possible.  As existing algorithms assume a fixed set of variables, one could break data sequences into fixed-length segments corresponding to each slice.  An SPN can be learned from this dataset of segments. However, it is not clear how to use the resulting SPN to construct a template network because a regular SPN has a single root while the template network has multiple roots and an equal number of input leaves that are not indicator variables.  One would have to treat each segment as independent data instances and could not answer queries about the probability of some variables in one slice given the values of other variables in other slices. 

We present an {\em anytime search-and-score} framework to learn the structure of the template SPN in a DSPN.  
It starts with an arbitrary structure and then generates several neighbouring structures.  It ranks the neighbouring structures according to a scoring function and selects the best neighbour.  These steps are repeated until a stopping criterion is met.  This framework can be instantiated in multiple ways based on the choice for the initial structure, the neighbour-generation process, the scoring function and the stopping criterion. We proceed with the description of a specific instantiation below, although other instantiations are possible.


Without loss of generality, we propose to use product nodes as the interface nodes for both the input and output of the template network.\footnote{WLOG assume that the DSPN alternates between layers of sum and product nodes. Since a DSPN consists of a repeated structure, there is flexibility in choosing the interfaces of the template.  We chose the interfaces to be at layers of product nodes, but the interfaces could be shifted by one level to layers of sum nodes or even traverse several layers to obtain a mixture of product and sum nodes.  These boundaries are all equivalent subject to suitable adjustments to the bottom and top networks.} We also propose to use a bottom network that is identical to the template network after removing the nodes that do not have any indicator variable as descendent.  This way we can design a single algorithm to learn the structure of the template network since the bottom network will be automatically determined from the learned template.  We also propose to fix the top network to a root sum node directly linked to all the input product nodes. For the template network, we initialize the SPN rooted at each output product node to a factored model of univariate distributions. Figure~\ref{fig:initial-structure} shows an example of this initial structure with two interface nodes and three variables.  Each output product node has four children where each child is a sum node corresponding to a univariate distribution.  Three of those children are univariate distributions linked to the indicators of the three variables, while the fourth sum node is a distribution over the interface input nodes. On merging the interface nodes for repeated instantiations of the template, we obtain a hierarchical mixture model. We begin with a single interface node and iteratively increase their number until the score 
stops improving.  Alg.~\ref{alg:initial-structure} summarizes the steps to compute the initial structure.



\begin{algorithm}[!ht]
\caption{Initial Structure}
\label{alg:initial-structure}
\begin{algorithmic} 
\REQUIRE $trainSet$, $validationSet$, $\tuple{X_1,...,X_n}$ (variables for a slice) 
\ENSURE $templNet$: Initial Template Network Structure
\STATE $g \leftarrow factoredDistribution(\tuple{X_1,...,X_n})$
\STATE $newTempl \leftarrow train(g,trainSet)$ 
\STATE \textbf{repeat} $templNet \leftarrow newTempl$; $newTempl \leftarrow train(templNet \cup \{g\},trainSet)$ 
\STATE \textbf{until} $likelihood(newTempl,validationSet) <likelihood(templNet,validationSet)$
\end{algorithmic}
\end{algorithm}

A simple scoring function is to use the likelihood of the data 
since exact inference in DSPNs can be done quickly. If the goal is to produce a generative model of the data, then the likelihood of the data is a natural criterion.  If the goal is to produce a discriminative model for classification, then the conditional likelihood of some class variables given the remaining variables is a suitable criterion.  For a given structure, parameters can be estimated using various parameter learning algorithms including gradient ascent~\citep{poon2011sum} and expectation maximization~\citep{poon2011sum,peharz2015foundations}.  

Our neighbour generation process (Alg.~\ref{alg:neighbour}) begins by sampling a product node uniformly and replacing the sub-SPN rooted at that product node by a new sub-SPN. Note that to satisfy the decomposability property, a product node must partition its scope into disjoint scopes for each of its children. Also note that different partitions of the scope can be seen as different conditional independencies between the variables~\citep{gens2013learning}. Hence, the search space of a product node generally corresponds to the set of all partitions of its scope.
We use the 'restricted growth string (RGS)' encoding of partitions to define a lexicographical order of the set of all possible partitions~\citep{Knuth:2006:ACP:1121689}. We can select the next partition according to the lexicographic ordering or by sampling from a distribution over all possible partitions. The distribution can be uniform in the absence of prior knowledge or an informed one otherwise. 

Since the search space is exponential in the number of variables in the scope of the product node, we greedily split the scope into mutually independent subsets according to pairwise independence tests applied recursively similar to \citep{gens2013learning}. 
In case no independent subsets are found, we sample a partition at random when the number of variables is greater than some threshold and select the next partition according to the lexicographic ordering otherwise.  Alg.~\ref{alg:get-partition} describes the process of finding the next partition based on which we construct a product of naive Bayes models (Fig.~\ref{fig:to-random-part-b}) where each naive Bayes model has two children that encode factored distributions.  This may increase or decrease the size of the template network depending on whether the new product of naive Bayes models replaces a bigger or smaller sub-SPN at the sampled product node.

Since constructing the new template, learning its parameters, and computing its score can be done in a time that is linear in the size of the template network and the dataset, each iteration of the anytime search-and-score algorithm scales linearly with the size of the template network and the amount of data.
\begin{algorithm}[!ht]
\caption{Generate neighbour (improved template network)}
\label{alg:neighbour}
\begin{algorithmic} 
\REQUIRE $trainSet, validationSet, templNet$ 
\ENSURE $templNet$
\REPEAT
	\STATE $n \leftarrow$ sample product node uniformly from $templNet$
	\STATE $newPartition \leftarrow \mathrm{GetPartition(n)}$
	\STATE $n' \leftarrow$ construct product of naiveBayes models based on $newPartition$
	\STATE $newTempl \leftarrow$ replace $n$ by $n'$ in $templNet$
\UNTIL $likelihood(newTempl,validationSet) <likelihood(templNet,validationSet)$
\end{algorithmic}
\end{algorithm}

\begin{algorithm}[!ht]
\caption{GetPartition}
\label{alg:get-partition}
\begin{algorithmic} 
\REQUIRE product node $n$ 
\ENSURE $nextPartition$
\STATE \textbf{if} $|\mathrm{scope(n)}| > \mathrm{threshold}$ \textbf{then}
	\STATE $\quad \{s_1, ..., s_k\} \gets$ partition $scope(n)$ into indep. subsets 
	\STATE $\quad$ \textbf{if} $k > 1$ \textbf{then return} $\cup_{i=1}^k GetPartition(s_i)$ 	
    \STATE $\quad$ \textbf{else return} random partition of $scope(n)$
\STATE \textbf{else return} next lexicographic partition of $scope(n)$ according to the RGS encoding 
\end{algorithmic}
\end{algorithm}


\begin{theorem}
The network templates produced by Alg.~\ref{alg:initial-structure} and~\ref{alg:neighbour} are invariant. 
\end{theorem}


\begin{proof}
Let the scope of all input interface nodes be identical.  The initial structure of the template network is a collection of factored distributions over all the variables.  Hence the output interface nodes all have the same scope (which includes all the variables).  Hence, Alg.~\ref{alg:initial-structure} produces an initial template network that is invariant. Alg.~\ref{alg:neighbour} replaces the sub-SPN of a product node by a new sub-SPN, which does not change the scope of the product node. This follows from the fact that the new partition used to construct the new sub-SPN has the same variables as the original partition.  Since the scope of the product node under which we change the sub-SPN does not change, all nodes above that product node, including the output interface nodes, preserve their scope. Hence Alg.~\ref{alg:neighbour} produces neighbour template networks that are invariant.
\end{proof}


\section{Experiments}
\label{sec:experiments}

We evaluate the performance of our anytime search-and-score method for DSPNs on several synthetic and real-world {\em sequence} datasets. In addition, we measure how well the DSPNs model the data by comparing the negative log-likelihoods with those of static SPNs learned using LearnSPN~\citep{gens2013learning}, and with other dynamic models such as Hidden Markov Models (HMM), DBNs and recurrent neural networks (RNNs). The threshold in Alg.~\ref{alg:get-partition} was set to 6 in all experiments.

The synthetic datasets include three dynamic processes with different structures: sequences of observations sampled from $(i)$ an HMM with one hidden variable, $(ii)$ the well-known Water DBN~\citep{jensen1989forprojekt} and $(iii)$ the Bayesian automated taxi (BAT) DBN~\citep{forbes1995batmobile}.  We also evaluate DSPNs with 5 real-world sequence datasets from the UCI repository~\citep{Lichman:2013}. They include applications such as online handwriting recognition~\citep{alimoglu1996methods} and speech recognition~\citep{hammami2009tree,kudo1999multidimensional}.  


We first compare DSPNs to the true model on the synthetic datasets.  As LearnSPN cannot be used with data of variable length, we include it in the synthetic datasets experiment only, where we sample sequences of fixed length.  Table~\ref{table:synthetic-data} shows the negative log-likelihoods based on 10-fold cross validation for the synthetic datasets. In all three synthetic datasets, DSPN learned generative models that exhibited likelihoods that are close to that of the true models. It also outperforms LearnSPN in all three cases. 

Next, we compare DSPNs to classic HMMs with parameters learned by Baum-Welch~\citep{baum1970maximization}, HMM-SPNs where each observation distribution is an SPN ~\citep{peharz2014modeling}, fully observable DBNs whose structure is learned by the Reveal algorithm~\citep{liang1998reveal} from the BayesNet Toolbox~\citep{Murphy01thebayes}, partially observable DBNs, whose structure and hidden variables are learned by search and score~\citep{friedman1998learning}, and RNNs with one input node, one hidden layer consisting of long short term memory (LSTM) units~\citep{hochreiter1997} and one output sigmoid unit with a cross-entropy loss function. We select LSTM units due to their popularity and success in sequence learning~\citep{SutskeverVL2014}. The input node corresponds to the value of the current observation and the output node to the predicted value of the next observation in the sequence. We train the network by backpropagation through time (bptt) truncated to 20 time steps~\citep{williams1990} with a learning rate of 0.01. Our implementation is based on the Theano library~\citep{theano} in Python.

\begin{table}[!t]
\centering
\begin{tabular}{|l|c|c|c|}
\hline
Dataset       		& HMM-Samples  &  Water	&  BAT	\\ 
(\#i, length, \#oVars) & (100, 100, 1)		& (100, 100, 4)		& (100, 100, 10) \\
\hline
True model 	& 62.2 $\pm$ 0.8 		& 249.6 $\pm$ 1.0 		& 628.2 $\pm$ 2.0	\\ 
\hline
LearnSPN	& 65.4 $\pm$ 0.7		& 270.4 $\pm$ 0.9		& 684.4 $\pm$ 1.3		\\ 
DSPN   		& {\bf 62.5} $\pm$ 0.7	& {\bf 252.4} $\pm$ 0.9	& {\bf 641.6} $\pm$ 1.1		\\ \hline
\end{tabular}
\caption{Mean negative log-likelihood and standard error based on 10-fold cross validation for the synthetic datasets. (\#i,length,\#oVars) indicates the number of data instances, length of each sequence and number of observed variables. Lower likelihoods are better.}
\label{table:synthetic-data}
\end{table}

\begingroup
\setlength{\tabcolsep}{3pt} 
\begin{table*}[]
\centering

\begin{tabular}{|l|c|c|c|c|c|}

\hline
Dataset & ozLevel & PenDigits & ArabicDigits & JapanVowels &  ViconPhysic \\
(\#i,length,\#oVars) & (2533,24,2)		& (10992,16,7)		& (8800,40,13)  & (640,16,12) & (200,3026,27) \\
\hline
HMM     			& 56.7 $\pm$ 1.1		& 74.2 $\pm$ 0.1    	& 327.5 $\pm$ 0.4    	& 94.3 $\pm$ 0.3 		& 40862 $\pm$ 369 \\
HMM-SPN    	    & 49.8 $\pm$ 0.9		& 67.7 $\pm$ 0.6    	& 305.8 $\pm$ 1.8    	& 89.8 $\pm$ 1.2    	& 38410 $\pm$ 440 \\
RNN        	    & {\bf 16.2} $\pm$ 0.7	& 68.7 $\pm$ 1.3    	& 303.6 $\pm$ 6.4    	& 78.8 $\pm$ 2.3    	& 
57217 $\pm$ 873 \\
Search-Score DBN  & 40.2 $\pm$ 4.7		& 67.3 $\pm$ 2.3      	& 263.7 $\pm$ 4.6      	& 75.6 $\pm$ 2.5        & - \\
Reveal DBN        & 52.4 $\pm$ 2.5		& 74.4 $\pm$ 0.2    	& 260.2 $\pm$ 1.0    	& 71.3 $\pm$ 1.2    	& - \\
DSPN    			& 33.0 $\pm$ 1.0		& {\bf 63.5} $\pm$ 0.3	& {\bf 257.9} $\pm$ 0.5 & {\bf 68.8} $\pm$ 0.3  & {\bf 36385} $\pm$ 682 \\
\hline    
\end{tabular}
\caption{Mean negative log-likelihood and standard error based on 10-fold cross validation for the real world datasets. (\#i,length,\#oVars) indicates the number of data instances, average length of the sequences and number of observed variables.
}
\label{table:exp-results}
\end{table*}
\endgroup
\begin{table}[b]
\centering
\resizebox{\columnwidth}{!}{
\begin{tabular}{|l|r|r|r|r|r|r|r|r|}
\hline
\multirow{3}{*}{Dataset} & \multicolumn{4}{c|}{Learning Time (Seconds)}                                                                                      & \multicolumn{4}{c|}{Inference Time (Seconds)}                                                                                                                                        \\ \cline{2-9} 
                         & \multicolumn{1}{c|}{\multirow{2}{*}{Reveal}} & \multicolumn{3}{c|}{Per Iteration}                                                 & \multicolumn{1}{c|}{\multirow{2}{*}{Reveal}} & \multicolumn{1}{l|}{\multirow{2}{*}{RNN}} & \multicolumn{1}{l|}{\multirow{2}{*}{SS DBN}} & \multicolumn{1}{l|}{\multirow{2}{*}{DSPN}} \\ \cline{3-5}
                         & \multicolumn{1}{c|}{}                        & \multicolumn{1}{l|}{RNN} & \multicolumn{1}{l|}{SS DBN} & \multicolumn{1}{l|}{DSPN} & \multicolumn{1}{c|}{}                        & \multicolumn{1}{l|}{}                     & \multicolumn{1}{l|}{}                        & \multicolumn{1}{l|}{}                      \\ \hline
ozLevel                  
	& 952                                          
    & 56                       
    & 108                         
    & 54                        
    & 6.3                                       
    & 0.1                                    
    & 15.6                                       
    & 0.1 
    \\
PenDigits                
	& 3,977                                         
    & 558                      
    & 1,463                        
    & 475                       
    & 15.0                                        
    & 0.2                                    
    & 30.7                                        
    & 0.1                                      
    \\
ArabicDigits             
	& 16,549                                        
    & 2572                
    & 14,911                       
    & 2,909                      
    & 53.6
    & 2.5                                     
    & 465.8                                        
    & 2.9                                       
    \\
JapaneseVowls            
	& 516                                          
    & 55                      
    & 363                         
    & 51                        
    & 15.2                                        
    & 0.2                                     
    & 69.2                                        
    & 0.5
    \\
ViconPhysical   
	& -                                            
    & 4705                        
    & -                           
    & 6734                        
    & -                                            
    & 2274                                       
    & -                                            
    & 1825
    \\ \hline
\end{tabular}
}

\caption{Comparisons of the learning and inference times of the networks learned by Reveal, RNN, Search-Score DBN (SS DBN) and DSPN.}
\label{time-table}

\end{table}

Table~\ref{table:exp-results} shows the results for the real datasets.  DSPNs outperform the other approaches except for one dataset where the RNN achieved better results.  DSPNs are more expressive than classic HMMs and HMM-SPNs since our search and score algorithm has the flexibility of learning a suitable structure with multiple interface nodes for the transition dynamics where as the structure of the transition dynamics is fixed with a single hidden variable in classic HMMs and HMM-SPNs.  DSPNs are also more expressive than the fully observable DBNs found by Reveal since the sum nodes in the template networks implicitly denote hidden variables.  DSPNs are as expressive as the partially observable DBNs found by search and score, but better results are achieved by DSPNs because their linear inference complexity allows us to explore a larger space of structures more quickly. DSPNs are less expressive than RNNs since DSPNs are restricted to sum and product nodes while RNNs use sum, product, max and sigmoid operators.  Nevertheless, RNNs are notoriously difficult to train due to the non-convexity of their loss function and vanishing/exploding gradient issues that arise in backpropagation through time.  This explains why RNNs did not outperform DSPNs on 4 of the 5 datasets.

Table~\ref{time-table} shows the time to learn and do inference with the DBN, RNN and DSPN models (the HMM models are omitted since they do not learn any structure for the transition dynamics and therefore are not as expressive).  All models were trained till convergence or up to two days. We report the total time for learning with Reveal and the time per iteration for learning with the other algorithms since they are anytime algorithms.  Learning DSPNs is generally faster than training RNNs and search-and-score DBNs. The time to do inference for all the sequences in each dataset when one variable is observed and the other variables are hidden is reported in the right hand side of the table.  DSPNs and RNNs are fast since they allow exact inference in linear time with respect to the size of their network, while the DBNs obtained by Reveal and search-and-score are slow because  inference may be exponential in the number of hidden variables if they all become correlated.

\section{Conclusion}
\label{sec:conclusion}

Existing methods for learning SPNs become inadequate when the task involves modeling sequence data such as time series data points. The specific challenge is that sequence data could be composed of instances of different lengths. Motivated by dynamic Bayesian networks, we presented a new model called dynamic SPN, which utilized a template network as a building block. We also defined a notion of invariance and showed that invariant template networks can be composed safely to ensure that the resulting DSPN is valid.  We provided an anytime algorithm based on the framework of search-and-score for learning the structure of the template network from data. As our experiments demonstrated, a DSPN fits sequential data better than static SPNs (produced by LearnSPN).  We also showed that the DSPNs found by our search-and-score algorithm achieve higher likelihood than competing HMMs, DBNs and RNNs on several temporal datasets. While approximate inference is typically used in DBNs to avoid an exponential blow up, inference can be done exactly in linear time with DSPNs.  


\appendix

\acks{This research was funded by a grant from Huawei Noah's Ark Lab in Hong Kong.}

\section{Proof of Theorem~\ref{thm:DSPN_valid}}
\label{appendix}

This supplementary material provides a detailed proof of Theorem~\ref{thm:DSPN_valid}.  We first introduce three lemmas and one corollary that are necessary to prove Theorem~\ref{thm:DSPN_valid}. 

Lemma~\ref{lemma:scope-union} shows that the scope of any node is the union of the input nodes of the subnetwork rooted at that node.  This will be useful in Lemma~\ref{lemma:relations} to show how the scope of different nodes relate to each other. 

\begin{lemma} [Scope Union] The scope of a node $i$ is the union of the scopes of the input nodes of the subnetwork rooted at $i$: 
\begin{equation}
scope(i) = \cup_{k \in inputs(i)} \; scope(k)
\label{eq:scope-union} 
\end{equation}
\label{lemma:scope-union} 
\end{lemma}

\begin{proof}
We give a proof by induction based on the level of each node.  For the base case, consider input leaf nodes (level 1).  Since an input node only has itself as input, it satisfies Eq.~\ref{eq:scope-union}.  For the induction step, assume that all nodes up to level $l$ satisfy Eq.~\ref{eq:scope-union}.  Since the scope of a node at level $l+1$ is the union of the scopes of its children at lower levels, then  
\begin{align}
scope(i) & = \cup_{child \in children(i)} \; scope(child) \\
&  = \cup_{child \in children(i)} \; [ \cup_{k \in inputs(child)} \; scope(k) ] \\
& = \cup_{k \in inputs(i)} \; scope(k)
\end{align}
\end{proof}

When the scopes of the input nodes of a network are either identical or disjoint then Lemma~\ref{lemma:relations} shows that changing the scopes of the input nodes in a way that preserves their identity and disjoint relations ensures that the identity and disjoint relations are also preserved for any pair of nodes in the network.  This will be useful in Corollary~\ref{corollary:decomposability-completeness} to show that completeness and decomposability are also preserved.

\begin{lemma} [Preservation of scope identity and disjoint relations] Let $g$ be a scope relabeling function that applies only to the scope of the input nodes.  If for all pairs $i$, $j$ of input nodes the following properties hold
\begin{align}
& scope(i) = scope(j) \vee scope(i) \cap scope(j) = \emptyset \label{eq:scope-id-disjoint} \\
& scope(i) = scope(j) \rightarrow g(scope(i)) = g(scope(j)) \label{eq:same-scope} \\
& scope(i) \cap scope(j) = \emptyset \rightarrow g(scope(j)) \cap g(scope(j)) = \emptyset  \label{eq:disjoint-scope}
\end{align}
then for all pairs $i$, $j$ of nodes the following properties hold
\begin{align}
& scope(i) = scope(j) \rightarrow scope_g(i) = scope_g(j) \label{eq:same-scope2} \\
& scope(i) \cap scope(j) = \emptyset \rightarrow scope_g(i) \cap scope_g(j) = \emptyset \label{eq:disjoint-scope2} 
\end{align}
Here $scope_g(i) = \cup_{j\in inputs(i)} \; g(scope(j))$ where $inputs(i)$ is the set of input nodes for the subnetwork rooted at $i$.
\label{lemma:relations}
\end{lemma}

\begin{proof}
Proof of Eq.~\ref{eq:same-scope2}: Suppose $scope(i) = scope(j)$ then 
\begin{equation}
\cup_{k \in inputs(i)} \; scope(k) = \cup_{l \in inputs(j)} \; scope(l) \mbox{  (by Lemma~\ref{lemma:scope-union})} 
\label{eq:equiv-inputs}
\end{equation}
Since the scope of each pair of inputs is either identical or disjoint (by Eq.~\ref{eq:scope-id-disjoint}), there exists a function $h$ that maps each input of $i$ to the set of inputs of $j$ with the same scope:
\begin{equation}
h(k) = \{l|scope(l)=scope(k), l \in inputs(j)\} \forall k \in inputs(i)
\label{eq:h-definition}
\end{equation}
Furthermore this function covers the inputs of $j$:
\begin{equation}
\cup_{k \in inputs(i)} \; h(k) = inputs(j)
\label{eq:h-coverage}
\end{equation}
We can then show that
\begin{align}
scope_g(i) & = \cup_{k \in inputs(i)} \; g(scope(k)) \\
& = \cup_{k \in inputs(i)} \; g(\cup_{l \in h(k)} \; scope(l)) \mbox{  (by Eq.~\ref{eq:h-definition})}  \\
& = \cup_{k \in inputs(i)} \; \cup_{l \in h(k)} \; g(scope(l)) \mbox{  (by Eq.~\ref{eq:same-scope})}  \\
& = \cup_{l \in inputs(j)} \; g(scope(l)) \mbox{  (by Eq.~\ref{eq:h-coverage})}\\
& = scope_g(j)
\end{align}
Proof of Eq.~\ref{eq:disjoint-scope2}: Suppose $scope(i) \cap scope(j) = \emptyset$ then
\begin{equation}
scope(k) \cap scope(l) = \emptyset \; \forall k \in inputs(i), l \in inputs(j) 
\label{eq:disjoint-inputs}
\end{equation}
We can then show that
\begin{align}
& scope_g(i) \cap scope_g(j) \\
& = (\cup_{k \in inputs(i)} \; g(scope(k))) \cap (\cup_{l \in inputs(j)} \; g(scope(l))) \\
& = \cup_{k \in inputs(i), l \in inputs(j)} \; g(scope(k)) \cap g(scope(l)) \\
& = \cup_{k \in inputs(i), l \in inputs(j)} \; \emptyset \mbox{  (by Eq.~\ref{eq:disjoint-inputs} and~\ref{eq:disjoint-scope})} \\
& = \emptyset
\end{align}
\end{proof}

When the scopes of the input nodes of a network are either identical or disjoint then Corollary~\ref{corollary:decomposability-completeness} shows that changing the scopes of the input nodes in a way that preserves their identity and disjoint relations ensures completeness and decomposability is preserved throughout the network.  This will be useful in Lemma~\ref{lemma:invariance} to show that composing multiple template networks preserves their invariance.

\begin{corollary} [Preservation of completeness and decomposability] 
Let $g$ be a scope relabeling function that applies only to the input nodes.  If for all pairs $i$, $j$ of input nodes the following properties hold
\begin{itemize}
\item $scope(i) = scope(j) \vee scope(i) \cap scope(j) = \emptyset$
\item $scope(i) = scope(j) \rightarrow g(scope(i)) = g(scope(j))$
\item $scope(i) \cap scope(j) = \emptyset \rightarrow g(scope(j)) \cap g(scope(j)) = \emptyset$
\end{itemize}
then decomposability and completeness are preserved.
\label{corollary:decomposability-completeness}  
\end{corollary}

\begin{proof}
According to Lemma~\ref{lemma:relations}, all pairs of nodes that have the same scope still have the same scope after relabeling the scopes with $g$.  Hence complete sum nodes (i.e., children all have the same scope) are still complete after relabeling the scopes with $g$.  Similary, according to Lemma~\ref{lemma:relations}, all pairs of nodes that have disjoint scopes still have the disjoint scopes after relabeling the scopes with $g$.  Hence decomposable product nodes (i.e., children have disjoint scopes) are still decomposable after relabeling the scopes with $g$. 
\end{proof}

When a template network is invariant, Lemma~\ref{lemma:invariance} shows that composing any number of template networks preserves invariance.  This result is the key to proving Theorem~\ref{thm:DSPN_valid}. 

\begin{lemma}
If a template network is invariant then a stack of arbitrarily many copies of this template network is also invariant.
\label{lemma:invariance}
\end{lemma}

\begin{proof}
We give a proof by induction based on the number of copies of the template network. For the base case, consider a stack of one copy of the template network.  Since the template network is invariant, then a stack of one copy of the template network is invariant.  For the induction step, assume that $n$ copies of the template network are invariant.  This means that there is a bijective function $f$ that maps each input of the first template to an output of the $n^{th}$ template such that for all pairs $i,j$ of inputs to the first template, the following properties hold:
\begin{align}
& scope(i) = scope(j) \iff scope(f(i)) = scope(f(j)) \\
& scope(i) \cap scope(j) = \emptyset \iff scope(f(i)) \cap scope(f(j)) = \emptyset \\
& scope(f(i)) = scope(f(j)) \vee scope(f(i)) \cap scope(f(j)) = \emptyset \label{eq:disjoint-or-identical}
\end{align}
Let $g$ be a function that maps the scope of each input $i$ of the first template to the scope of the output of the $n^{th}$ template according to $f$:
\begin{equation}
g(scope(i)) = scope(f(i)) 
\end{equation}
Since assigning the scopes of the output nodes of the bottom network to the input nodes of the $n+1^{th}$ template ensures that the $n+1^{th}$ template is complete and decomposable and $g$ can be viewed as a relabeling of those scopes, then by Eq.~\ref{eq:disjoint-or-identical}, Lemma~\ref{lemma:relations} and Corollary~\ref{corollary:decomposability-completeness}, the $n+1^{th}$ template is also invariant. As a result, the entire stack of $n+1$ templates is invariant.
\end{proof}

We are now ready to prove the main theorem.

\vspace{0.5cm}
\noindent {\bf Theorem~\ref{thm:DSPN_valid}}
{\em If (a) the bottom network is complete and decomposable, (b) the scopes of all pairs of output interface nodes of the bottom network are either identical or disjoint, (c) the scopes of the output interface nodes of the bottom network can be used to assign scopes to the input interface nodes of the template and top networks in such a way that the template network is invariant and the top network is complete and decomposable, then the corresponding DSPN is complete and decomposable.}
\vspace{0.5cm}

\begin{proof}
The bottom network is complete and decomposable by assumption.  Since we also assume that the scopes of all pairs of the output interface nodes of the bottom network are either identical or disjoint and  the output interface nodes of the bottom network can be used to assign scopes to the interface nodes of the template network, then by Lemma~\ref{lemma:invariance} a stack of any number of template networks is invariant (and therefore complete and decomposable).  Finally we show that the top network is also complete and decomposable.  Let $f$ be a bijective function that associates each input of the first template to an output of the last template such that for all pairs $i,j$ of inputs to the first template, the following properties hold:
\begin{align}
& scope(i) = scope(j) \iff scope(f(i)) = scope(f(j)) \\
& scope(i) \cap scope(j) = \emptyset \iff scope(f(i)) \cap scope(f(j)) = \emptyset \\
& scope(f(i)) = scope(f(j)) \vee scope(f(i)) \cap scope(f(j)) = \emptyset \label{eq:disjoint-or-identical2}
\end{align}
Let $g$ be a function that maps the scope of each input $i$ of the first template to the scope of the output of the last template according to $f$:
\begin{equation}
g(scope(i)) = scope(f(i)) 
\end{equation}
Since assigning the scopes of the output nodes of the bottom network to the input nodes of the top network ensures that the top network is complete and decomposable and $g$ can be viewed as a relabeling of those scopes, then by Eq.~\ref{eq:disjoint-or-identical} and Corollary~\ref{corollary:decomposability-completeness}, the top network is complete and decomposable.
\end{proof}

\begin{small}
\bibliography{reference}
\end{small}
\end{document}